\documentclass[11pt,a4paper]{article}
\usepackage[T1]{fontenc}
\usepackage[utf8]{inputenc}
\usepackage{lmodern}
\usepackage[margin=25mm]{geometry}
\usepackage{amsmath,amssymb,amsthm}
\usepackage{graphicx}
\usepackage{authblk}
\usepackage{cite}
\usepackage[hyphens]{url}
\usepackage{tikz}
\usepackage[margin=0pt]{subfig}
\usetikzlibrary{3d, arrows.meta}
\usepackage{pgfplots}
\usepackage[hidelinks]{hyperref}
\pgfplotsset{compat=1.18}
\DeclareMathOperator{\proj}{proj}
\DeclareMathOperator{\dist}{dist}

\theoremstyle{plain}
\newtheorem{theorem}{Theorem}[section]
\newtheorem{lemma}[theorem]{Lemma}

\theoremstyle{definition}
\newtheorem{definition}[theorem]{Definition}
\newtheorem{example}[theorem]{Example}
\newtheorem{question}{Question}

\title{Relocation of compact sets in $\mathbb{R}^n$ by diffeomorphisms and linear separability of datasets in $\mathbb{R}^n$}
\author[1,2]{Xiao-Song Yang\thanks{Corresponding author. Email: \href{mailto:yangxs@hust.edu.cn}{\texttt{yangxs@hust.edu.cn}}.}}
\author[1]{Xuan Zhou\thanks{Email: \href{mailto:xuanzhou1037@hust.edu.cn}{\texttt{xuanzhou1037@hust.edu.cn}}.}}
\author[1]{Qi Zhou\thanks{Email: \href{mailto:qizhou1037@hust.edu.cn}{\texttt{qizhou1037@hust.edu.cn}}.}}
\affil[1]{School of Mathematics and Statistics, Huazhong University of Science and Technology, 1037 Luoyu Road, Wuhan 430074, P.R. China}
\affil[2]{Hubei Key Laboratory of Engineering Modeling and Scientific Computing, Huazhong University of Science and Technology, Wuhan 430074, P.R. China}

\date{}
\hypersetup{
  pdftitle={Relocation of compact sets in R\textasciicircum n by diffeomorphisms and linear separability of datasets in R\textasciicircum n},
  pdfauthor={Xiao-Song Yang, Xuan Zhou, Qi Zhou}
}

\begin{document}
\maketitle

	\begin{abstract}		
		Relocation of compact sets in an $n$-dimensional manifold by self-diffeomorphism is of its own interest as well as significant potential applications to data classification in data science. This paper presents a theory for relocating a finite number of compact sets in $\mathbb{R}^n$ to be relocated to arbitrary target domains in $\mathbb{R}^n$ by diffeomorphisms of $\mathbb{R}^n$. Furthermore, we prove that for any such collection, there exists a differentiable embedding into $\mathbb{R}^{n+1}$ such that their images become linearly separable.
		
		As applications of the established theory, we show that a finite number of compact datasets in $\mathbb{R}^n$ can be made linearly separable by width-$n$ deep neural networks (DNNs) with Leaky-ReLU, ELU, or SELU activation functions, under a mild condition. In addition, we show that any finite number of mutually disjoint compact datasets in $\mathbb{R}^n$ can be made linearly separable in $\mathbb{R}^{n+1}$ by a width-$(n+1)$ DNN.
	\end{abstract}
	
\par\medskip
	
	\section{Introduction}\label{Introduction}
	Differential topology is a fascinating branch of mathematics that has extensive applications in multiple other research fields. In this paper, we are interested in the following questions because of their potential application in deep learning theory:
	
	\begin{question} \label{q:general_relocation}
		Let $M$ be an $n$-dimensional smooth manifold. Let $D_1, \dots, D_k \subset M$ be mutually disjoint compact subsets, and $C_1, \dots, C_k \subset M$ be a collection of disjoint balls (or target domains). Under what conditions does there exist a diffeomorphism $F: M \to M$ such that $F(D_i) \subset C_i$ for all $i=1, \dots, k$?
	\end{question}
	
	\begin{question} \label{q:linear_separability}
		In the Euclidean case $M=\mathbb{R}^n$, does there exist a diffeomorphism $F: \mathbb{R}^n \to \mathbb{R}^n$ such that the images $F(D_1), \dots, F(D_k)$ are linearly separable (\textup{i.e.}, separable by $(n-1)$-dimensional disjoint hyperplanes)?
	\end{question}
	
	\begin{question} \label{q:embedding_width}
		More generally, if the sets $D_1$ and $D_2$ cannot be made linearly separable by any self-diffeomorphism of $\mathbb{R}^n$, what is the minimal dimension $m > n$ such that there exists an embedding $F: \mathbb{R}^n \to \mathbb{R}^m$ for which $F(D_1)$ and $F(D_2)$ become linearly separable in $\mathbb{R}^m$?
	\end{question}
	
	It is known in brain science that a collection of neural population called ``object manifold'' response to an object sensed by visual system~\cite{cohen2020separability, grootswagers2019untangling}, and different object manifold responses to different sensed object. It has also been hypothesized that the visual hierarchy untangles ``object manifolds'' into linearly separable ones in the setting of deep neural networks. This naturally motivates the questions asked in the beginning from perspective of diffeomorphism theory.
	
	In data science, the support vector machine \cite{braga2020fundamentals} is an important technique for data classification in the situation that the datasets are linearly separable. Nonetheless, datasets in general are not linearly separable in Euclidean space, thus it is natural to try to make them linearly separable by some transformations between Euclidean spaces without loss of the information of the original datasets. This poses the question of if it is possible to have a diffeomorphism that can transform the datasets to linear separable ones in the sense that datasets with different labels can be separated by hyperplanes in the Euclidean space.
	
	In addition, it is interesting in its own right to study how a self-diffeomorphism of a smooth manifold manipulates compact subsets of the manifold in the context of differential topology.
	
	The purpose of this paper is to answer the above questions by establishing a theory for the relocation of compact subsets within a smooth manifold via diffeomorphisms, and to apply this theory to problems of linear separability of geometric objects in deep learning. One of the main results in this paper is that for a finite number of mutually disjoint compact datasets in $\mathbb{R}^n$, there exists a width-$(n+1)$ deep neural network (DNN) $\Phi: \mathbb{R}^n \to \mathbb{R}^{n+1}$ with Leaky-ReLU, ELU, or SELU activation function, making these compact sets linearly separable in $\mathbb{R}^{n+1}$.
	
	The remainder of this paper is organized as follows. Section \ref{Preliminaries and Notation} introduces the preliminary concepts and notations of differential topology and deep neural network architectures. Section \ref{Relocation of Compact Sets via Diffeomorphisms} establishes the core theoretical framework for the relocation of compact sets, rigorously proving that disjoint compact subsets can be relocated into arbitrary target domains by ambient diffeomorphisms. Section \ref{Linear separability by DNNs} applies these topological results to deep learning, demonstrating that deep narrow networks with specific activation functions can untangle complex data structures to achieve linear classifiability, such as the Hopf Link and the Swiss Roll.
	
	\section{Preliminaries and Notation}\label{Preliminaries and Notation}
	In this section, we introduce notations and definitions used throughout this paper.
	
	\subsection{Basic concepts of Differential Topology}\label{Basic concepts of Differential Topology}
	\begin{definition}[Standard ball]\label{open ball}
		Let $B^n(\bar{x}, r) = \{ x \in \mathbb{R}^n \mid \lVert x - \bar{x}\rVert < r \}$ denote the \textbf{Standard ball} in $\mathbb{R}^n$ with center $\bar{x}$ and radius $r$.
	\end{definition}
	
	\begin{definition}[Smooth map on Compact Sets]\label{def:smooth_map}
		For a compact set $D \subset \mathbb{R}^n$, a map $f: D \to \mathbb{R}^n$ is said to be \textbf{smooth} if there exists an open neighborhood $U$ of $D$ and a smooth map $F: U \to \mathbb{R}^n$ such that $F|_D = f$.
	\end{definition}
	
	\begin{definition}[Differentiable $k$-ball]\label{def:k_ball}
		A subset $D \subset \mathbb{R}^n$ is called a \textbf{differentiable $k$-ball} or \textbf{$k$-ball} (where $k \le n$) if there exists an open neighborhood $U$ of the closed unit ball $\overline{B^k} \subset \mathbb{R}^k$ and a smooth embedding $\psi: U \to \mathbb{R}^n$ such that $D = \psi(\overline{B^k})$.
	\end{definition}
	
	\begin{definition}[Diffeomorphisms]\label{def:diff_spaces}
		Let $U \subset \mathbb{R}^d$ be an open subset, and let $r$ be a non-negative integer or infinity.
		\begin{itemize}
			\item $\mathcal{D}^r(U)$ denotes the set of $C^r$-diffeomorphisms from $U$ to its image in $\mathbb{R}^d$.
			\item $\text{Diff}_c^r(\mathbb{R}^d)$ denotes the set of compactly supported $C^r$-diffeomorphisms on $\mathbb{R}^d$. A diffeomorphism $f$ on $\mathbb{R}^d$ is \textbf{compactly supported} if there exists a compact subset $K\subset \mathbb{R}^d$ such that for any $x\notin K$, $f(x)=x$.
		\end{itemize}
	\end{definition}
	
	\subsection{Neural Network Architectures}\label{Neural Network Architectures}
	In this subsection, we review some of the fundamental definitions of DNNs.	
	\begin{definition}[Activation Functions]\label{Activation Functions}
		We consider the following activation functions mentioned in this paper:
		\begin{itemize}
			\item \textbf{ReLU} (Rectified Linear Unit): $\mathrm{ReLU}(x) = \max\{0, x\}$.
			
			\item \textbf{Leaky-ReLU}: For a parameter $0 < \alpha < 1$,
			\[
			\mathrm{Leaky\textup{-}ReLU}_{\alpha}(x) = \begin{cases} x & \textup{if } x \ge 0, \\ \alpha x & \textup{if } x < 0. \end{cases}
			\]
			
			\item \textbf{ELU} (Exponential Linear Unit): For a parameter $\alpha > 0$,
			\[
			\mathrm{ELU}_{\alpha}(x) = \begin{cases} x & \textup{if } x \ge 0, \\ \alpha(e^x - 1) & \textup{if } x < 0. \end{cases}
			\]
			
			\item \textbf{SELU} (Scaled Exponential Linear Unit): For parameters $\lambda > 0$ and $\alpha > 0$,
			\[
			\mathrm{SELU}_{\lambda, \alpha}(x) = \lambda \begin{cases} x & \textup{if } x \ge 0, \\ \alpha(e^x - 1) & \textup{if } x < 0. \end{cases}
			\]
		\end{itemize}
		
		\noindent Let $\sigma$ denote any of the aforementioned activation functions. Activation functions applied to vector function as componentwise operators. For $x = (x_1, \dots, x_d) \in \mathbb{R}^d$,
		\[
		\sigma(x) := (\sigma(x_1), \dots, \sigma(x_d)).
		\]
	\end{definition}
	
	\begin{definition}[Deep Neural Network and Width]\label{Deep Neural Network and Width}
		A Deep Neural Network (DNN) $\Phi: \mathbb{R}^{n_{\mathrm{in}}} \to \mathbb{R}^{n_{\mathrm{out}}}$ with $L$ hidden layers is a mapping constructed by the composition of affine transformations and non-linear activation functions, expressed as
		\[
		\Phi(x) = T_L \circ \sigma \circ T_{L-1} \circ \dots \circ \sigma \circ T_1(x),
		\]
		where
		\begin{itemize}
			\item $T_i: \mathbb{R}^{d_{i-1}} \to \mathbb{R}^{d_i}$ are affine maps defined by $T_i(y) = W_i y + b_i$, with weight matrices $W_i \in \mathbb{R}^{d_i \times d_{i-1}}$ and bias vectors $b_i \in \mathbb{R}^{d_i}$.
			\item $\sigma: \mathbb{R} \to \mathbb{R}$ is a non-linear activation function.
			\item The dimensions of the layers are given by the sequence $(d_0, d_1, \dots, d_L)$, where $d_0 = n_{\mathrm{in}}$ and $d_L = n_{\mathrm{out}}$.
		\end{itemize}
		
		The \textbf{width} of the DNN, denoted by $w(\Phi)$, is defined as the maximum dimension of its hidden layers,
		\[
		w(\Phi) = \max \{d_1, d_2, \dots, d_{L-1}\}.
		\]
		Throughout this paper, the term ``width-$n$ DNN'' refers to a network where the dimension of every hidden layer is at most $n$ (\textup{i.e.}, $w(\Phi) \le n$).
	\end{definition}
	
	\begin{definition}[Compact Approximation]\label{Compact Approximation}
		Let $\mathcal{F}$ and $\mathcal{G}$ be two sets of functions from $X$ to $Y$. We say that $\mathcal{G}$ \textbf{compactly approximates} $\mathcal{F}$ if for any $f \in \mathcal{F}$, any compact set $K \subset X$, and any $\epsilon > 0$, there exists a function $g \in \mathcal{G}$ such that
		\[
		\sup_{x \in K} \lVert f(x) - g(x)\rVert < \epsilon.
		\]
	\end{definition}
	
	\section{Relocation of Compact Sets by Diffeomorphisms}\label{Relocation of Compact Sets via Diffeomorphisms}
	In this section, we establish a theory for the relocation of compact subsets within a connected smooth manifold $M$. For convenience, we only consider the Euclidean space $\mathbb{R}^n$ in the subsequent analysis.
	
	The following lemma constructs a smooth isotopy centered at the origin of $\mathbb{R}^{n}$, capable of compressing a ball to an arbitrarily small radius while remaining the identity map outside a fixed neighborhood.
	\begin{lemma}\label{isotopy compression lemma}
		For any $0 < \delta < r$, there exists a smooth isotopy $F: [0, 1] \times \mathbb{R}^n \to \mathbb{R}^n$ such that
		\begin{enumerate}
			\item $F(s, \cdot): \mathbb{R}^n \to \mathbb{R}^n$ is a diffeomorphism for each $s$, and $F(0, \cdot) = \mathrm{id}$.
			\item $F(s, x) = x$ for all $x \in \mathbb{R}^n \setminus B^n(0, \theta)$, for some fixed $\theta > r$.
			\item $F(1, B^n(0, r)) \subset B^n(0, \delta)$.
		\end{enumerate}
	\end{lemma}
	
	\begin{proof}
		Fix a real number $\theta > r$. Let $\eta: \mathbb{R}^n \to [0, 1]$ be a smooth bump function such that $\eta(x) = 1$ for $x \in B^n(0, r)$ and $\eta(x) = 0$ for $x \in \mathbb{R}^n \setminus B^n(0, \theta)$. Construct a smooth vector field $V$ on $\mathbb{R}^n$ defined by
		\[
		V(x) = -\eta(x)x,\quad x \in \mathbb{R}^n.
		\]
		Let $\phi: \mathbb{R} \times \mathbb{R}^n \to \mathbb{R}^n$ be the global flow generated by the vector field $V$, satisfying
		\[
		\frac{\partial}{\partial t}\phi(t, x) = V(\phi(t, x)), \quad \phi(0, x) = x.
		\]
		Since $V$ is a smooth vector field with compact support contained in $\overline{B^n(0, \theta)}$, the flow $\phi$ is well-defined for all time $t \in \mathbb{R}$, and each map $\phi(t, \cdot)$ is a diffeomorphism of $\mathbb{R}^n$. 
		By the properties of the bump function $\eta$, the solution to the differential equation is given by
		\[
		\phi(t, x) = 
		\begin{cases} 
			e^{-t}x & \textup{if } x \in B^n(0, r), \\
			x & \textup{if } x \in \mathbb{R}^n \setminus B^n(0, \theta).
		\end{cases}
		\]
		It is clear that for $t$ large enough, specifically if we choose $\bar{t} > 0$ such that $e^{-\bar{t}}r < \delta$, that is, $\bar{t} > \ln(r/\delta)$, we have
		\[
		\phi(\bar{t}, B^n(0, r)) \subset B^n(0, \delta).
		\]
		Finally, we define the required isotopy $F: [0, 1] \times \mathbb{R}^n \to \mathbb{R}^n$ by setting $F(s, x) = \phi(s\bar{t}, x)$. It follows immediately from the standard properties of the flow and the definition of $\bar{t}$ that $F$ satisfies conditions (1)--(3).
	\end{proof}
	
	Having established the ability to compress a ball in Lemma \ref{isotopy compression lemma}, we now demonstrate that points can be transported within a local domain. The following lemma asserts that within a convex neighborhood, any two points can be connected by a diffeomorphism that acts strictly as a translation along the connecting segment, while remaining the identity outside the neighborhood.
	
	\begin{lemma}\label{Homogeneous lamma}
		Given two points $p, q \in B^n(\bar{x}, r) \subset \mathbb{R}^n$, there exists a diffeomorphism $H: \mathbb{R}^n \to \mathbb{R}^n$ such that $H(p) = q$ and $H(x) = x$ for all $x \in \mathbb{R}^n \setminus B^n(\bar{x}, r)$.
	\end{lemma}
	
	\begin{proof}
		Let $\rho = \max\{\lVert p-\bar{x}\rVert, \lVert q-\bar{x}\rVert\}$. Since $p, q \in B^n(\bar{x}, r)$, we have $\rho < r$. 
		Choose a constant $\delta$ such that $\rho < \delta < r$.
		
		Let $\eta: [0, \infty) \to [0, 1]$ be a smooth bump function such that $\eta(s) = 1$ for $s \in [0, \delta^2]$. Now, we construct a smooth vector field $V$ on $\mathbb{R}^n$ defined by
		\[
		V(x) = \eta(\lVert x-\bar{x}\rVert^2)(q-p), \quad x \in \mathbb{R}^n.
		\]
		This vector field is compactly supported within the open ball $B^n(\bar{x}, r)$. Let $\phi: \mathbb{R} \times \mathbb{R}^n \to \mathbb{R}^n$ be the global flow generated by $V$, which is the solution to the Cauchy problem
		\[
		\frac{\partial}{\partial t}\phi(t, x) = V(\phi(t, x)), \quad \phi(0, x) = x.
		\]
		The solution satisfies the equivalent integral equation
		\[
		\phi(t, x) = x + \int_0^t \eta(\lVert\phi(\tau, x)-\bar{x}\rVert^2)(q - p) \, d\tau.
		\]
		We define the diffeomorphism $H: \mathbb{R}^n \to \mathbb{R}^n$ by $H(x) = \phi(1, x)$. For any $x \in \mathbb{R}^n \setminus B^n(\bar{x}, r)$, $V(x) = 0$, we have $H(x) = x$. 
		
		Now, consider the line segment $L$ connecting $p$ and $q$. Since $p, q \in B^n(\bar{x}, \delta)$ and the ball is convex, the entire segment $L$ lies within $B^n(\bar{x}, \delta)$, where $\eta(\lVert z-\bar{x}\rVert^2) \equiv 1$. Consequently, the integral equation simplifies to $\phi(t, p) = p + t(q-p)$ for $t \in [0, 1]$, which gives $H(p) = q$.		
	\end{proof}
	
	By means of the vector field constructed in Lemma \ref{isotopy compression lemma}, we can define a corresponding vector field on the image of an open ball by a smooth embedding. This allows us to compress an embedded open ball, located arbitrarily in the space, into a small neighborhood of its center.
	
	\begin{lemma}\label{Transitivity Lemma}
		Let $g: B^n(0, 1) \to \mathbb{R}^n$ be an embedding, and let $D = g(B^n(0, 1))$. Let $p = g(0)$.
		For any ball $B^n(p, \delta) \subset D$ and $r \in (0, 1)$, there exists a diffeomorphism $H: \mathbb{R}^n \to \mathbb{R}^n$ satisfying
		\begin{enumerate}
			\item $H(g(B^n(0, r))) \subset B^n(p, \delta)$,
			\item $H(x) = x$ for all $x \in \mathbb{R}^n \setminus D$.
		\end{enumerate}
	\end{lemma}
	
	\begin{proof}
		In view of Lemma \ref{isotopy compression lemma}, we first choose a constant $r_0$ such that $r < r_0 < 1$. Let $\eta: \mathbb{R}^n \to [0, 1]$ be a smooth cutoff function such that $\eta \equiv 1$ on $\overline{B^n(0, r)}$ and $\operatorname{supp}(\eta) \subset B^n(0, r_0)$. We define a smooth vector field $V$ on the domain $\mathbb{R}^n$ by
		\[
		V(x) = -\eta(x)x, \quad x \in \mathbb{R}^n.
		\]
		
		Now, we define the vector field $V_g$ on $\mathbb{R}^n$ as follows:
		\[
		V_g(y) = 
		\begin{cases} 
			Dg(g^{-1}(y)) \cdot V(g^{-1}(y)) & \textup{if } y \in D, \\
			0 & \textup{if } y \in \mathbb{R}^n \setminus D.
		\end{cases}
		\]
		Since $V = 0$ near the boundary of $B^n(0, 1)$, $V_g = 0$ near the boundary of $D$ and extends smoothly to zero on all of $\mathbb{R}^n$.
		
		Let $\Phi: \mathbb{R} \times \mathbb{R}^n \to \mathbb{R}^n$ be the flow of the vector field $V_g$.
		The flow $\Phi$ on the image $D$ corresponds to the flow $\phi$ of $V$ on the domain via the embedding $g$; thus, we have
		\[
		\Phi(t, g(x)) = g(\phi(t, x)), \quad \text{for } x \in B^n(0, 1).
		\]
		For points $x \in B^n(0, r)$, the flow on the domain is explicitly given by $\phi(t, x) = e^{-t}x$. As $t \to \infty$, the set $\phi(t, B^n(0, r)) = B^n(0, e^{-t}r)$ shrinks to the origin $0$. By the continuity of $g$, the image $g(\phi(t, B^n(0, r)))$ shrinks to the point $p$. Therefore, for a sufficiently large $\bar{t} > 0$,
		\[
		\Phi(\bar{t}, g(B^n(0, r))) = g(\phi(\bar{t}, B^n(0, r))) \subset B^n(p, \delta).
		\]
		Then, the map $H(y) = \Phi(\bar{t}, y)$ is the required diffeomorphism satisfying conditions (1) and (2).
	\end{proof}	
	
	In view of the compression techniques (Lemmas \ref{isotopy compression lemma} and \ref{Transitivity Lemma}) and the local relocation method (Lemma \ref{Homogeneous lamma}) established above, we now demonstrate that a subset can be transported between arbitrary locations without intersecting other fixed sets. The following lemma guarantees the existence of such a transformation by constructing a ``safety tube'' along a path, and applying a finite sequence of local translations to achieve this.
	
	\begin{lemma}\label{Homogeneous of complement}
		Let $K \subset \mathbb{R}^n$ be a compact subset. Let $p, q \in \mathbb{R}^n \setminus K$. Suppose that $\mathbb{R}^n \setminus K$ is path-connected.
		Then, for any $\delta_2 > 0$, there exists a radius $\delta_1 > 0$ and a diffeomorphism $H: \mathbb{R}^n \to \mathbb{R}^n$ such that
		\begin{enumerate}
			\item $H(B^n(p, \delta_1)) \subset B^n(q, \delta_2)$.
			\item $H(x) = x$ for all $x \in K$.
		\end{enumerate}
	\end{lemma}	
	
	\begin{proof}
		Since $\mathbb{R}^n \setminus K$ is path-connected, let $\alpha: [0, 1] \to \mathbb{R}^n \setminus K$ be a continuous path with $\alpha(0) = p$ and $\alpha(1) = q$. The image $\Gamma = \alpha([0, 1])$ is compact. Since $K$ is closed and disjoint from $\Gamma$, the distance between them is strictly positive. Let
		\[
		2\rho = \dist(\Gamma, K) = \inf \{ \lVert x - y\rVert \mid x \in \Gamma, y \in K \} > 0.
		\]
		We construct a ``chain of balls'' to move $p$ to $q$. For each point $z \in \Gamma$, consider the open ball $B(z, \rho)$. The collection $\{B(z, \rho)\}_{z \in \Gamma}$ is an open cover of the path $\Gamma$. By compactness of $\Gamma$, there exists a finite subcover $B_1, B_2, \dots, B_k$ where each $B_i = B(z_i, \rho)$. We order these balls such that $p \in B_1$, $q \in B_k$, and $B_i \cap B_{i+1} \neq \emptyset$.
		
		We choose a sequence of points $x_0, \dots, x_k$ such that $x_0 = p$, $x_k = q$, and $x_j \in B_j \cap B_{j+1}$ for $1 \le j < k$. By applying Lemma \ref{Homogeneous lamma} iteratively, for each $i \in \{1, \dots, k\}$, since both $x_{i-1}$ and $x_i$ lie in $B_i$, there exists a diffeomorphism $h_i$ supported in $B_i$ such that $h_i(x_{i-1}) = x_i$. 
		
		We define the composite diffeomorphism $H: \mathbb{R}^n \to \mathbb{R}^n$ by
		\[
		H = h_k \circ h_{k-1} \circ \dots \circ h_1.
		\]
		By construction, it follows immediately that $H(p) = x_k = q$.
		Furthermore, since $B_i \cap K = \emptyset$ by the choice of $\rho$, each map restricts to the identity on $K$; thus, $H|_K = \mathrm{id}$.
		
		In view of Lemma \ref{Homogeneous lamma}, each local diffeomorphism $h_i$ is constructed to act as a translation in a neighborhood of the line segment connecting $x_{i-1}$ and $x_i$. Since the composition of finite translations is a translation, there exists an open neighborhood $U$ of $p$ such that the restriction $H|_U$ acts as the translation $y \mapsto y + (q-p)$. This implies that $H$ is a local isometry on $U$. Since $U$ is open, we can choose a source radius $\delta_1 > 0$ sufficiently small such that $B^n(p, \delta_1) \subset U$ and $\delta_1 \le \delta_2$. It follows that $H$ maps the ball $B^n(p, \delta_1)$ onto $B^n(q, \delta_1)$, which is contained in $B^n(q, \delta_2)$. This completes the proof.
	\end{proof}
	
	In view of the previous lemmas, we now state and prove the main topological result of this section. By combining the compression capability established in Lemma \ref{Transitivity Lemma} and the transportation established in Lemma \ref{Homogeneous of complement}, we demonstrate that any finite collection of disjoint compact sets contained in disjoint open balls can be relocated into arbitrary target open balls.
	
	\begin{theorem}\label{thm:disjoint_mapping}
		Given compact subsets $K_1, \dots, K_m$ in $\mathbb{R}^n$. Assume there exist mutually disjoint balls $D_1, \dots, D_m \subset \mathbb{R}^n$ such that $K_i \subset D_i$ for $i=1, \dots, m$, and $D_i \cap D_j = \emptyset$ for $i \neq j$. Then, for any balls $B^n(x_i, r_i) \subset \mathbb{R}^n\setminus \bigcup_{j=1}^{m} D_j$ ($i=1, \dots, m$), there exists an ambient diffeomorphism $H : \mathbb{R}^n \to \mathbb{R}^n$, such that
		\[
		H(K_i) \subset B^n(x_i, r_i), \quad \textup{for all } i=1, \dots, m.
		\]
	\end{theorem}
	
	\begin{proof}
		For each $i \in \{1, \dots, m\}$, since $D_i$ is a ball, let $D_i = f_i(\overline{B^n(0, 1)})$ for some smooth embedding $f_i$. Let $p_i = f_i(0)$ be the center of $D_i$.
		
		Consider a sufficiently small ball $B^n(p_i, \bar{r}_i) \subset D_i$ where $\bar{r}_i$ is small enough.
		By Lemma \ref{Transitivity Lemma}, there exists a diffeomorphism $H_i: \mathbb{R}^n \to \mathbb{R}^n$ such that $H_i(x) = x$ for $x \in \mathbb{R}^n \setminus D_i$ and $H_i(K_i) \subset B^n(p_i, \bar{r}_i)$.
		
		Note that since $K_i \subset D_i$, we can always find an intermediate set $\tilde{D}_i = f_i(B^n(0, \theta_i))$ with $\theta_i < 1$ such that $K_i \subset \tilde{D}_i$, and the compression maps $\tilde{D}_i$ into the small ball near $p_i$. Since the domains $D_1, \dots, D_m$ are mutually disjoint ($D_i \cap D_j = \emptyset$), the support of $H_i$ is disjoint from the support of $H_j$. We can define a global compression map $\hat{H}$ as the composition
		\[
		\hat{H} = H_m \circ H_{m-1} \circ \dots \circ H_1.
		\]
		Applying this map to the subsets, we have
		\[
		\hat{H}(K_i) = H_i(K_i) \subset B^n(p_i, \bar{r}_i), \quad \text{for } i=1, \dots, m.
		\]
		
		Now we need to move the compressed sets from $p_i$ to the target centers $x_i$. We construct a sequence of diffeomorphisms $H^i$ iteratively. Fix $i \in \{1, \dots, m\}$.
		To guarantee that the map $H^i$ restricts to the identity on $D_j$ for $j > i$ and on $B^n(x_j, r_j)$ for $j < i$, we define the compact set
		\[
		\mathcal{O}_i = \bigg( \bigcup_{j=i+1}^m \overline{D_j} \bigg) \cup \bigg( \bigcup_{j=1}^{i-1} \overline{B^n(x_j, r_j)} \bigg).
		\]
		Since $n \ge 2$, the space $\mathbb{R}^n \setminus \mathcal{O}_i$ is path-connected. The points $p_i$ and $x_i$ lie in this space. By Lemma \ref{Homogeneous of complement}, there exists a radius $\delta_i > 0$ and a diffeomorphism $H^i: \mathbb{R}^n \to \mathbb{R}^n$ fixing $\mathcal{O}_i$ (\textup{i.e.}, $H^i(x) = x$ for $x \in \mathcal{O}_i$) such that
		\[
		H^i(B^n(p_i, \delta_i)) \subseteq B^n(x_i, r_i).
		\]
		
		We choose the compression radius $\bar{r}_i$ such that $\bar{r}_i \le \min\big\{\delta_i, \frac{r_i}{2}\big\}$, and let
		\[
		\tilde{H} = H^m \circ \dots \circ H^1,
		\]
		we have
		\[
		\tilde{H}(B^n(p_i, \bar{r}_i)) \subseteq H^i(B^n(p_i, \delta_i)) \subseteq B^n(x_i, r_i), \quad \text{for } i=1, \dots, m.
		\]
		
		Finally, let the final diffeomorphism $H$ be the composition of compression and transportation
		\[
		H = \tilde{H} \circ \hat{H}.
		\]
		We have
		\[
		H(K_i) = \tilde{H}(\hat{H}(K_i)) \subset \tilde{H}(B^n(p_i, \bar{r}_i)) \subset B^n(x_i, r_i). \qedhere
		\]
	\end{proof}
	
	In fact, we have the following more general statement that gives an affirmative answer to Question \ref{q:embedding_width} in the introduction.
	
	\begin{theorem}\label{thm:embedding_rn1}
		Let $K_1, \dots, K_m$ be mutually disjoint compact sets in $\mathbb{R}^n$. Then, for any ball $B$ containing $\bigcup_{i=1}^m K_i$ in its interior, there exists a smooth embedding $E: B \to \mathbb{R}^{n+1}$ such that the images $E(K_1), \dots, E(K_m)$ can be relocated to arbitrary target domains in $\mathbb{R}^{n+1}$ by diffeomorphisms of $\mathbb{R}^{n+1}$.
	\end{theorem}
	
	\begin{proof}
		The proof is straightforward; the reader can complete the proof by virtue of Urysohn's Lemma and the approximation of continuous maps by smooth ones.
	\end{proof}
	
	\section{Linear separability by DNNs}\label{Linear separability by DNNs}
	As mentioned in the introduction, linear separability is a fundamental concept in machine learning, essential to both linear support vector machines and deep learning. While dealing with data that is nonlinearly separable, some methods such as the kernel method and feature engineering have been developed to make the data linearly separable. Deep neural networks (DNNs) are now the main method to transform nonlinearly separable data into linearly separable data; thus, how network structures like width and depth affect the performance of the designed DNNs is an active topic in the investigation of the theoretical foundation of DNNs.
	
	In this section, we study the minimum width of DNNs capable of transforming nonlinear separability to linear separability as applications of the results established in Section \ref{Relocation of Compact Sets via Diffeomorphisms}. Before discussing these applications, however, we introduce a fundamental result regarding the approximation capability of DNNs. Specifically, Hanin and Sellke \cite{hanin2017approximating} established the approximation of continuous functions using minimal-width ReLU networks, and Kidger and Lyons \cite{kidger2020universal} extended this to general non-polynomial activation functions. Subsequently, Teshima \textup{et al.} \cite{teshima2020coupling} showed that any diffeomorphism can be approximated by a composition of single coordinate transformations. Based on these works, Hwang \cite{hwang2023minimum} formulated the precise minimum width required for DNNs to uniformly approximate ambient diffeomorphisms. We now state this final form by Hwang \cite{hwang2023minimum}, which provides the crucial theoretical link allowing us to translate differential topological existence theorems into neural network constructions.	
	
	\begin{theorem}\label{thm:nn_is_invertible_approximator}
		Let $\sigma$ be a continuous function that is a $C^{1}$-function near $\alpha \in \mathbb{R}$ with $\sigma^{\prime}(\alpha) \ne 0$. Then, for a natural number $d \in \mathbb{N}$, the class of deep neural networks with input and output dimensions $d$ and width $d+\alpha(\sigma)$ is capable of approximating any $C^2$-diffeomorphism of $\mathbb{R}^d$ uniformly on any compact set, where 
		\[
		\alpha(\sigma) = \begin{cases}
			0 & \textup{if } \sigma = \textup{Leaky-ReLU}, \\ 
			1 & \textup{if } \sigma = \textup{ReLU}, \\
			2 & \textup{otherwise}.
		\end{cases}
		\]
	\end{theorem}
	
	Recently, Yang \textup{et al.} \cite{yang2025minimum} extended this universal approximation property to other activation functions, proving that deep neural networks with ELU or SELU activation functions have the same approximation capabilities as those with Leaky-ReLU. Consequently, the minimum width guarantee for approximating ambient diffeomorphisms naturally holds for ELU networks as well.
	
	Building upon the results of Hwang \cite{hwang2023minimum} and Yang \textup{et al.} \cite{yang2025minimum}, we now establish the neural network realization of our main topological result, Theorem \ref{thm:disjoint_mapping}. The following lemma confirms that a deep neural network of width $n$ with either Leaky-ReLU, ELU, or SELU activation function is sufficient to relocate disjoint compact datasets into arbitrary target open balls.
	
	\begin{lemma}\label{thm:DNN_relocation}
		Let $K_1, \dots, K_m$ be compact datasets in $\mathbb{R}^n$ (where $n \ge 2$) satisfying the assumptions of Theorem \ref{thm:disjoint_mapping} (\textup{i.e.}, they are contained in mutually disjoint balls).
		Then, for any target balls $B^n(x_i, r_i)$ disjoint from the initial domains, there exists a width-$n$ Deep Neural Network (DNN) using Leaky-ReLU, ELU, or SELU activation function, such that
		\[
		\Phi(K_i) \subset B^n(x_i, r_i), \quad \textup{for all } i=1, \dots, m.
		\]
	\end{lemma}
	
	\begin{proof}
		By Theorem \ref{thm:disjoint_mapping}, there exists a global diffeomorphism $H: \mathbb{R}^n \to \mathbb{R}^n$ such that $H(K_i) \subset B^n(x_i, r_i)$ for all $i=1, \dots, m$.
		
		The set $H(K_i)$ is a compact subset strictly contained in the open ball $B^n(x_i, r_i)$. Let $\mu$ be the minimum distance between the boundary $\partial B^n(x_i, r_i)$ and the set $H(K_i)$, given by
		\[
		\mu = \min_{i=1,\dots,m} \dist(H(K_i), \partial B^n(x_i, r_i)) > 0.
		\]
		
		By virtue of Theorem \ref{thm:nn_is_invertible_approximator}, for any $\epsilon < \mu$, there exists a width-$n$ DNN $\Phi$ with Leaky-ReLU, ELU, or SELU activation function that uniformly approximates the diffeomorphism $H$ on the compact set $\bigcup_{i=1}^m K_i$.
		Specifically, we can find $\Phi$ satisfying
		\[
		\lVert\Phi(x) - H(x)\rVert < \epsilon, \quad \textup{for all } x \in \bigcup_{i=1}^m K_i.
		\]
		Therefore, we have
		\[
		\Phi(K_i) \subset B^n(x_i, r_i), \quad \textup{for all } i=1, \dots, m. \qedhere
		\]
	\end{proof}
	
	A direct consequence of Lemma \ref{thm:DNN_relocation} is that a non-linearly separable dataset can be transformed into a linearly classifiable one. By mapping data points associated with distinct labels into mutually disjoint open balls arranged in a linearly separable configuration, the dataset becomes linearly classifiable by a neural network.
	
	\begin{theorem}\label{thm:linear_classifiability}
		Let $D \subset \mathbb{R}^n$ (where $n \ge 2$) be a compact dataset with $l$ distinct labels. Assume there exist mutually disjoint open balls $D_1, \dots, D_m \subset \mathbb{R}^n$ ($m \ge l$) such that
		\begin{enumerate}
			\item $D \subset \bigcup_{i=1}^{m} D_i$.
			\item For each $i$, the subset $D \cap D_i$ contains data of only one label.
		\end{enumerate}
		Then, there exists a width-$n$ Deep Neural Network (DNN) $\Psi: \mathbb{R}^n \to \mathbb{R}^n$ using Leaky-ReLU, ELU, or SELU activation function such that the transformed dataset $\Psi(D)$ is linearly classifiable.
	\end{theorem}
	
	\begin{proof}
		Let the distinct labels be denoted by $j \in \{1, \dots, l\}$. We first construct $l$ mutually disjoint target open balls corresponding to these labels. Since the finite union of source domains $\bigcup_{k=1}^m D_k$ is bounded, we can choose a base point $p_1 \in \mathbb{R}^n$ sufficiently far from the origin such that $B^n(p_1, 1) \cap D_k = \emptyset$ for all $k$. Consider a vector $v \in \mathbb{R}^n$ with $\lVert v\rVert > 2$. Define a sequence of centers $p_j$ along the direction of $v$ as follows:
		\[
		p_j = p_1 + (j-1)v, \quad j=1, \dots, l.
		\]
		Consider the balls $B_j = B^n(p_j, 1)$ for $j=1, \dots, l$. Clearly all $B_j$ are mutually disjoint and can be classified by a linear function $L: \mathbb{R}^n \to \mathbb{R}$. 
		
		Now, we assign each input ball $D_i$ to a target region based on its label. For each $i \in \{1, \dots, m\}$, let $y_i \in \{1, \dots, l\}$ be the label associated with compact subset $K_i = D \cap D_i$. For any label $j$, if there are multiple input balls $D_i$ having label $j$, we can choose enough small disjoint sub-balls within $B_j$ as their specific targets.
		
		Thus, for each $i=1, \dots, m$, we can assign a unique target ball $T_i$ such that
		\[
		T_i \subset B_{y_i} \subset \mathbb{R}^n \setminus \bigcup_{k=1}^m D_k.
		\]
		By Lemma \ref{thm:DNN_relocation}, there exists a width-$n$ DNN $\Psi$ such that
		\[
		\Psi(K_i) \subset T_i \subset B_{y_i}.
		\]
		Consequently, all data points with label $j$ are mapped into $B_j$. Since the sets $\{B_1, \dots, B_l\}$ are linearly separable, the transformed dataset $\Psi(D)$ is linearly classifiable.
	\end{proof}
	
	While the previous results focus on relocating disjoint compact sets, many data structures in machine learning are modeled as low-dimensional manifolds embedded in a high-dimensional space. To approximate smooth embeddings of such manifolds via ambient diffeomorphisms, we first require a theoretical guarantee that a local embedding of a ball can be extended to a global diffeomorphism of the ambient space. This is provided by the following lemma derived from Palais \cite{palais1960extending}.
	
	\begin{lemma}[Extending Diffeomorphism]\label{lem:palais_extension}
		Let $D \subset \mathbb{R}^m$ be a compact differentiable $k$-ball. Suppose $g: U \to \mathbb{R}^m$ is a smooth embedding defined on an open neighborhood $U$ of $D$. 
		Then, there exists a global diffeomorphism $G: \mathbb{R}^m \to \mathbb{R}^m$ such that
		\[
		G(x) = g(x), \quad \forall x \in D.
		\]
	\end{lemma}
	
	Consider a differentiable $k$-ball in $M$ in terms of Definition \ref{def:k_ball}. By combining the extending diffeomorphism guaranteed by Lemma \ref{lem:palais_extension} with the universal approximation property of Theorem \ref{thm:nn_is_invertible_approximator}, we can now establish that DNNs are capable of approximating smooth embeddings of differentiable balls. This result provides a theoretical foundation for manifold learning tasks, such as unfolding low-dimensional data manifolds.
	
	\begin{theorem}\label{thm:embedding_approx}
		Let $D \subset \mathbb{R}^n$ be a $k$-dimensional differentiable ball (where $k < n$), and let $\phi: D \to \mathbb{R}^n$ be a smooth embedding.
		Then, for any given $\epsilon > 0$, there exists a width-$n$ Deep Neural Network (DNN) $\Phi$ using Leaky-ReLU, ELU, or SELU activation function, such that
		\[
		\lVert \Phi(x) - \phi(x) \rVert < \epsilon, \quad \forall x \in D.
		\]
	\end{theorem}
	
	\begin{proof}
		Since $D$ is a $k$-dimensional differentiable ball, it is a compact subset of $\mathbb{R}^n$. By the Definition \ref{def:smooth_map} of a smooth map on a compact set, there exists an open neighborhood $U$ of $D$ and a smooth map $\tilde{\phi}: U \to \mathbb{R}^n$ such that $\tilde{\phi}|_D = \phi$. Since $\phi$ is an embedding on the compact set $D$, there exists a possibly smaller open neighborhood $V \subset U$ containing $D$ such that $\tilde{\phi}|_V$ is a smooth embedding.
		
		By Lemma \ref{lem:palais_extension}, there exists a global diffeomorphism $F: \mathbb{R}^n \to \mathbb{R}^n$ that extends $\phi$, \textup{i.e.},
		\[
		F(x) = \phi(x), \quad \forall x \in D.
		\]
		
		According to Theorem \ref{thm:nn_is_invertible_approximator}, which states that the set of width-$n$ Deep Neural Networks (with Leaky-ReLU, ELU, or SELU activation functions) compactly approximates $\mathcal{D}^2(\mathbb{R}^n)$. Since $D$ is a compact set, for the given precision $\epsilon > 0$, there exists a width-$n$ DNN $\Phi: \mathbb{R}^n \to \mathbb{R}^n$ such that
		\[
		\lVert \Phi(x) - F(x) \rVert < \epsilon, \quad \forall x \in D.
		\]
		Since $F(x) = \phi(x)$ for all $x \in D$, it immediately follows that
		\[
		\lVert \Phi(x) - \phi(x) \rVert < \epsilon, \quad \forall x \in D.
		\]
		This completes the proof.
	\end{proof}
	
	While Theorem \ref{thm:embedding_approx} is restricted to embeddings, generalizing this approximation to arbitrary smooth maps requires overcoming potential topological obstructions, such as self-intersections. The following theorem resolves this by employing a dimension-lifting technique: decomposing the general smooth map into a higher-dimensional embedding followed by a canonical projection. This establishes that DNNs can uniformly approximate any smooth map on a compact set.
	
	\begin{theorem}\label{thm:projection}
		Let $D \subset \mathbb{R}^n$ be a compact set, and $f: D \to \mathbb{R}^n$ be a smooth map.
		Let $B^n \subset \mathbb{R}^n$ be an $n$-ball containing $D$. Suppose there exists an integer $m \ge n$ and a smooth embedding $g: D \to \mathbb{R}^m$ that can be extended to $B^n$, such that
		\[
		f(x) = \proj \circ g(x), \quad \forall x \in D,
		\]
		where $\proj: \mathbb{R}^m \to \mathbb{R}^n$ is the standard projection from $\mathbb{R}^m$ to $\mathbb{R}^n$.
		
		Then, for any $\epsilon > 0$, there exists a width-$m$ Deep Neural Network (DNN) $\Phi: \mathbb{R}^n \to \mathbb{R}^n$ with Leaky-ReLU, ELU, or SELU activation function such that
		\[
		\lVert \Phi(x) - f(x) \rVert < \epsilon, \quad \forall x \in D.
		\]
	\end{theorem}
	
	\begin{proof}
		Let us regard $\mathbb{R}^n$ as a subspace of $\mathbb{R}^m$, \textup{i.e.}, $\mathbb{R}^m = \mathbb{R}^n \times \mathbb{R}^{m-n}$. Thus we have the inclusion
		\[
		D \subset B^n \subset \mathbb{R}^n \times \{0\} \subset \mathbb{R}^m.
		\]
		
		By hypothesis, we have an extending embedding $\tilde{g}: B^n \to \mathbb{R}^m$.
		According to Lemma \ref{lem:palais_extension}, since $B^n$ is topologically a trivial ball, this embedding $\tilde{g}$ can be extended to a global diffeomorphism of the ambient space $\mathbb{R}^m$.
		Let $G: \mathbb{R}^m \to \mathbb{R}^m$ be this diffeomorphism such that
		\[
		G|_{B^n} = \tilde{g}.
		\]
		It follows that $G|_D = g$.
		
		By Theorem \ref{thm:nn_is_invertible_approximator} of Hwang \cite{hwang2023minimum}, for any $\epsilon > 0$, there exists a DNN $\Psi$ of width $m$ with Leaky-ReLU, ELU, or SELU activation functions such that
		\[
		\lVert \Psi(z) - G(z) \rVert < \epsilon, \quad \forall z \in B^n.
		\]
		
		Construct the final network $\Phi$ by composing the projection with $\Psi$, given by
		\[
		\Phi(x) = \proj \circ\Psi(x).
		\]
		For any $x \in D$, we have
		\begin{align*}
			\lVert \Phi(x) - f(x) \rVert &= \lVert \proj(\Psi(x)) - \proj(g(x)) \rVert \\
			&= \lVert \proj(\Psi(x) - G(x)) \rVert\\
			&\le \lVert \proj \rVert \cdot \lVert \Psi(x) - G(x) \rVert < \epsilon.
		\end{align*}
		This completes the proof.
	\end{proof}
	
	To illustrate the practical implications of the established theoretical framework, we present three concrete examples. The first example provides an intuitive illustration of Theorem \ref{thm:linear_classifiability} regarding the linear classifiability of disjoint compact sets. The second example utilizes the dimension-lifting technique of Theorem \ref{thm:projection} to overcome topological obstructions, rendering linked sets linearly separable. The final example illustrates the approximation of smooth embeddings as established in Theorem \ref{thm:embedding_approx}.
	
	\begin{example}\label{ex:toy_obstruction}
		To illustrate the practical implications of Theorem \ref{thm:linear_classifiability}, we consider a toy example of datasets in $\mathbb{R}^2$, as shown in Figure \ref{fig:toy_example}. The dataset consists of three distinct classes: two inner datasets, labeled by $A$ and $B$, enclosed by an outer ring-shaped dataset, labeled by $C$. In the Euclidean plane $\mathbb{R}^2$, it is easy to see that there exists no global diffeomorphism $\Psi: \mathbb{R}^2 \to \mathbb{R}^2$ that can make these datasets linearly separable. This is because any global diffeomorphism preserves topological invariants. Specifically, the enclosure relationship is preserved. Since $C$ topologically encloses $A$ and $B$, its image $\Psi(C)$ must necessarily enclose $\Psi(A)$ and $\Psi(B)$. Consequently, it is impossible to construct any separating straight line in $\mathbb{R}^2$ to separate $\Psi(C)$ from $\Psi(A)$ or $\Psi(B)$. Thus, the topological obstruction is vital and cannot be overcome by any self-diffeomorphism of the plane. It is also apparent that the labeled sets cannot be separated by three pairwise disjoint balls, so Theorem \ref{thm:disjoint_mapping} and Lemma \ref{thm:DNN_relocation} cannot apply.
		
		However, the three datasets can be separated by three balls in $\mathbb{R}^3$ after linearly embedding them into $\mathbb{R}^3$. To see this, assume for convenience that the plane containing these datasets is the linear subspace $\mathbb{R}^2 \times \{0\}$ (i.e., the $xy$-plane). Then we can construct three balls in $\mathbb{R}^3$ that are pair-wise disjoint and contain the sets $A$, $B$, and $C$, as illustrated in Figure \ref{fig:3d_torus_protruding_cells} and Figure \ref{fig:2d_ushape_cross_section}. Thus, Theorem \ref{thm:linear_classifiability} can be applied to the case where the ambient space is $\mathbb{R}^3$.
		
		\begin{figure}[tbp]
			\centering 
			\begin{tikzpicture}
				\newcommand\Rout{3.6}  
				\newcommand\Rin{2.2}   
				\newcommand\rAB{0.7}   
				\newcommand\posX{1.0}  
				
				\filldraw[fill=cyan!20, draw=black!80, thick, even odd rule] 
				(0,0) circle (\Rout)
				(0,0) circle (\Rin);
				
				\node at (0, 2.9) {\Large \textbf{C}};
				
				\filldraw[fill=violet!30, draw=black!80, thick] 
				(-\posX, 0) circle (\rAB);
				\node at (-\posX, 0) {\Large \textbf{A}};
				
				\filldraw[fill=orange!30, draw=black!80, thick] 
				(\posX, 0) circle (\rAB);
				\node at (\posX, 0) {\Large \textbf{B}};
				
			\end{tikzpicture}
			\caption{Three distinct datasets in $\mathbb{R}^2$.} 
			\label{fig:toy_example} 
		\end{figure}
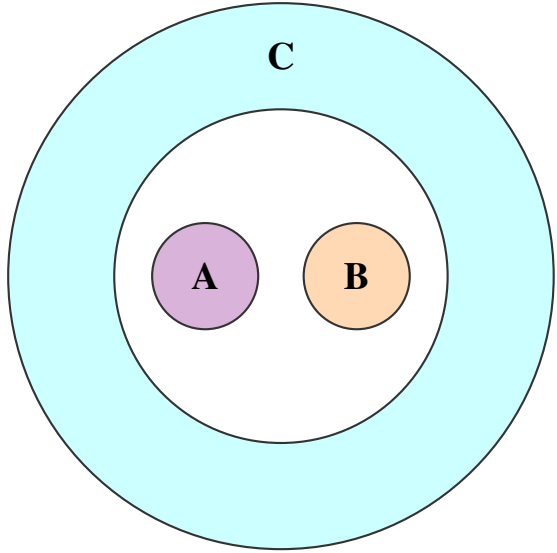
		
		\begin{figure}[tbp]
			\centering
			\subfloat[A full 3D perspective. A, B, and C are separated by mutually disjoint balls in $\mathbb{R}^3$.\label{fig:3d_torus_protruding_cells}]{%
				\resizebox{0.48\textwidth}{!}{%
					\begin{tikzpicture}[scale=1.0]
						\newcommand\Rout{5.4}   
						\newcommand\RoutY{2.2}  
						\newcommand\Zout{3.4}   
						\newcommand\Rin{2.3}    
						\newcommand\RinY{0.9}
						\newcommand\cRout{4.5}  
						\newcommand\cRoutY{1.8}
						\newcommand\cRin{3.1}   
						\newcommand\cRinY{1.2}
						\newcommand\Zup{1.4}    
						\newcommand\Rmid{3.85}  
						\newcommand\RmidY{1.55} 
						\newcommand\rx{1.55}    
						
						\draw[dashed, black!30, thick] (-7.5, 3.8) -- (7.5, 3.8) -- (7.5, -4.5) -- (-7.5, -4.5) -- cycle;
						\node[black!50, right] at (7.5, -4.5) {$xOy$ Plane};
						
						\filldraw[fill=cyan!30, draw=cyan!70!black, thick, opacity=0.35]
						(-\Rout, 0) arc [start angle=180, end angle=360, x radius=\Rout, y radius=\Zout]
						arc [start angle=360, end angle=180, x radius=\Rout, y radius=\RoutY];
						
						\shade[inner color=cyan!50, outer color=cyan!5] 
						(0,0) ellipse [x radius=\Rin, y radius=\RinY];
						\draw[cyan!60!black, thick] 
						(0,0) ellipse [x radius=\Rin, y radius=\RinY];
						
						\filldraw[fill=cyan!40, draw=cyan!80!black, thick, opacity=0.85, even odd rule]
						(0,0) ellipse [x radius=\cRout, y radius=\cRoutY]
						(0,0) ellipse [x radius=\cRin, y radius=\cRinY];
						
						\node[black] at (-3.8, 0) {\Large \textbf{C}};
						\node[black] at (3.8, 0) {\Large \textbf{C}};
						
						\coordinate (posA) at (-1.15, 0.1);
						\coordinate (posB) at (1.15, 0.1);
						
						\filldraw[fill=violet!50, draw=violet!80!black, thick] (posA) ellipse [x radius=0.35, y radius=0.14];
						\node[white] at (posA) {\textbf{A}};
						\shade[ball color=violet!30, opacity=0.4] (posA) circle (0.8);
						\draw[violet!80!black, dashed, opacity=0.8] (posA) circle (0.8);
						\node[violet!80!black, above] at (-1.15, 1.2) {\Large \textbf{Ball $B_A$}};
						
						\filldraw[fill=orange!50, draw=orange!80!black, thick] (posB) ellipse [x radius=0.35, y radius=0.14];
						\node[white] at (posB) {\textbf{B}};
						\shade[ball color=orange!30, opacity=0.4] (posB) circle (0.8);
						\draw[orange!80!black, dashed, opacity=0.8] (posB) circle (0.8);
						\node[orange!80!black, above] at (1.15, 1.2) {\Large \textbf{Ball $B_B$}};
						
						\fill[cyan!20, opacity=0.35, even odd rule]
						(0,0) ellipse [x radius=\Rout, y radius=\RoutY]
						(0,0) ellipse [x radius=\Rin, y radius=\RinY];
						
						\draw[cyan!50!black, thick, dashed, opacity=0.6] 
						(0, \Zup) ellipse [x radius=\Rmid, y radius=\RmidY];
						
						\draw[cyan!60!black, thick, opacity=0.6]
						(-\Rout, 0) arc [start angle=180, end angle=0, x radius=\rx, y radius=\Zup];
						\draw[cyan!60!black, thick, opacity=0.6]
						(\Rin, 0) arc [start angle=180, end angle=0, x radius=\rx, y radius=\Zup];
						
						\draw[cyan!60!black, thick, opacity=0.4]
						(0, -\RoutY) .. controls (0, -\RoutY + \Zup*0.9) and (0, -\RinY + \Zup*0.9) .. (0, -\RinY);
						\draw[cyan!60!black, thick, dashed, opacity=0.4]
						(0, \RoutY) .. controls (0, \RoutY + \Zup*0.9) and (0, \RinY + \Zup*0.9) .. (0, \RinY);
						
						\draw[cyan!70!black, thick, opacity=0.6] (0,0) ellipse [x radius=\Rout, y radius=\RoutY];
						\draw[cyan!70!black, thick, opacity=0.6] (0,0) ellipse [x radius=\Rin, y radius=\RinY];
						
						\node[cyan!80!black] at (-4.5, -2.6) {\Huge \textbf{Ball $B_C$}};
					\end{tikzpicture}%
				}%
			}
			\hfill 
			\subfloat[Cross-sectional view. A, B, and C are separated by mutually disjoint balls in $\mathbb{R}^3$.\label{fig:2d_ushape_cross_section}]{%
				\resizebox{0.48\textwidth}{!}{%
					\begin{tikzpicture}[scale=1.2]
						\filldraw[fill=cyan!10, draw=black!80, thick]
						(-4.5, 1.5)
						-- (-4.5, 0) arc (180:360:4.5 and 2.5) 
						-- (4.5, 1.5)
						arc (0:180:0.75)                       
						-- (3.0, 0) arc (360:180:3.0 and 1.0)  
						-- (-3.0, 1.5)
						arc (0:180:0.75) -- cycle;             
						
						\filldraw[fill=violet!10, draw=black!80, thick] (-1.2, 1.5) circle (0.8);
						\filldraw[fill=orange!10, draw=black!80, thick] (1.2, 1.5) circle (0.8);
						
						\draw[cyan!90!black, line width=3pt, line cap=round] (-4.2, 1.5) -- (-3.3, 1.5);
						\node[above] at (-3.75, 1.6) {\Large \textbf{C}};
						
						\draw[violet!90!black, line width=3pt, line cap=round] (-1.6, 1.5) -- (-0.8, 1.5);
						\node[above] at (-1.2, 1.6) {\Large \textbf{A}};
						
						\draw[orange!90!black, line width=3pt, line cap=round] (0.8, 1.5) -- (1.6, 1.5);
						\node[above] at (1.2, 1.6) {\Large \textbf{B}};
						
						\draw[cyan!90!black, line width=3pt, line cap=round] (3.3, 1.5) -- (4.2, 1.5);
						\node[above] at (3.75, 1.6) {\Large \textbf{C}};
						
						\node[cyan!80!black] at (0, -2.0) {\large Ball $B_C$};
						\node[violet!80!black] at (-1.2, 0.4) {\large Ball $B_A$};
						\node[orange!80!black] at (1.2, 0.4) {\large Ball $B_B$};
					\end{tikzpicture}%
				}%
			}
			\caption{Visualizing the separation in higher dimensions.}
		\end{figure}
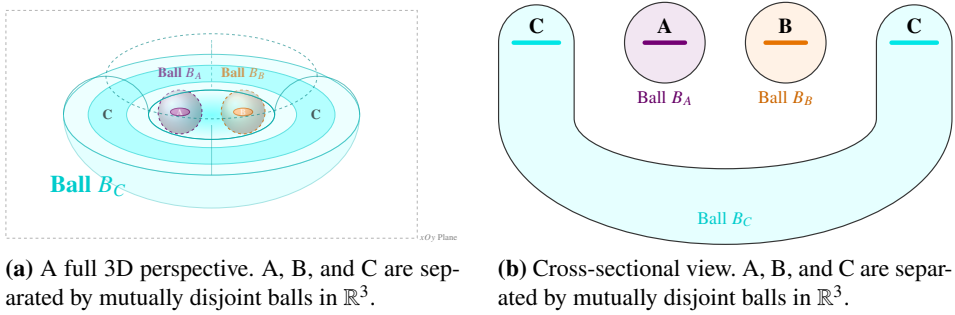
	\end{example}
	
	\noindent\textbf{Experimental Simulation.} As stated in the theoretical analysis of Example \ref{ex:toy_obstruction}, embedding the datasets into $\mathbb{R}^3$ implies that it suffices to employ a width-$3$ DNN. To empirically validate this theoretical guarantee, we constructed and trained a width-$3$ DNN with a Leaky-ReLU activation function. Specifically, we define the input regions $A$, $B$, and $C$ as follows:
	\begin{align*}
		A &= \{(x, y) \mid (x+1)^2+(y-1)^2 \leq 1\}, \\
		B &= \{(x, y) \mid (x-1)^2+(y+1)^2 \leq 1\}, \\
		C &= \{(x, y) \mid 9 \leq x^2+y^2 \leq 25\}.
	\end{align*}
	
	\begin{figure}[tbp]
		\centering
		\subfloat[Original 2D datasets (INPUT)\label{fig:ex20_input}]{%
			\includegraphics[width=0.45\textwidth]{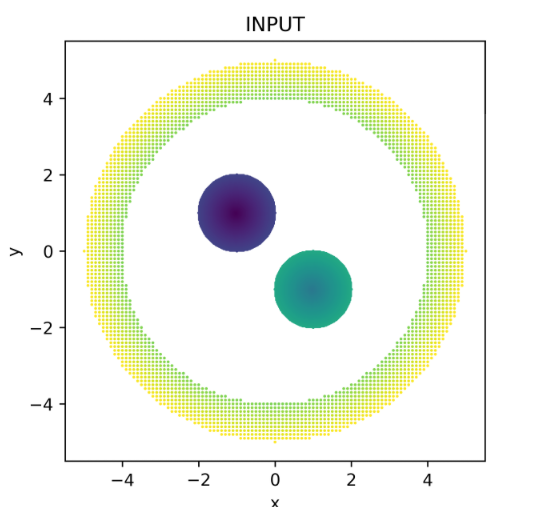}%
		}\hfill
		\subfloat[Linear embedding into $\mathbb{R}^3$\label{fig:ex20_embed}]{%
			\includegraphics[width=0.45\textwidth]{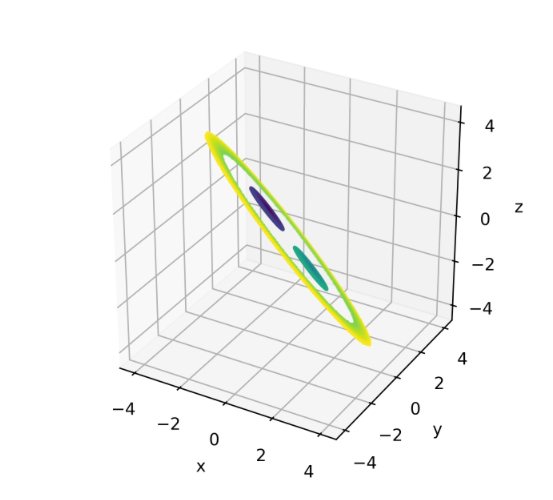}%
		}
		
		\subfloat[3D deformation by the DNN\label{fig:ex20_deform}]{%
			\includegraphics[width=0.45\textwidth]{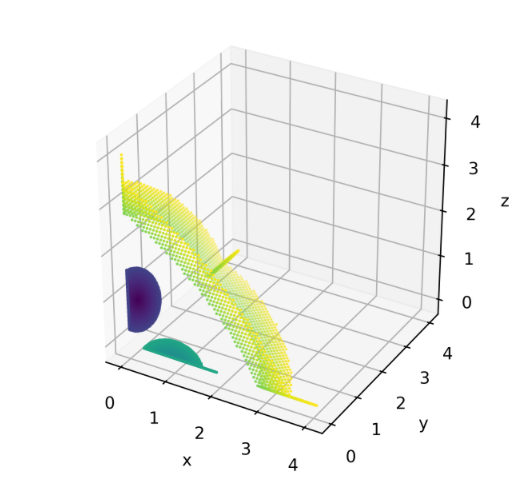}%
		}\hfill
		\subfloat[Linearly separable state (OUTPUT)\label{fig:ex20_output}]{%
			\includegraphics[width=0.45\textwidth]{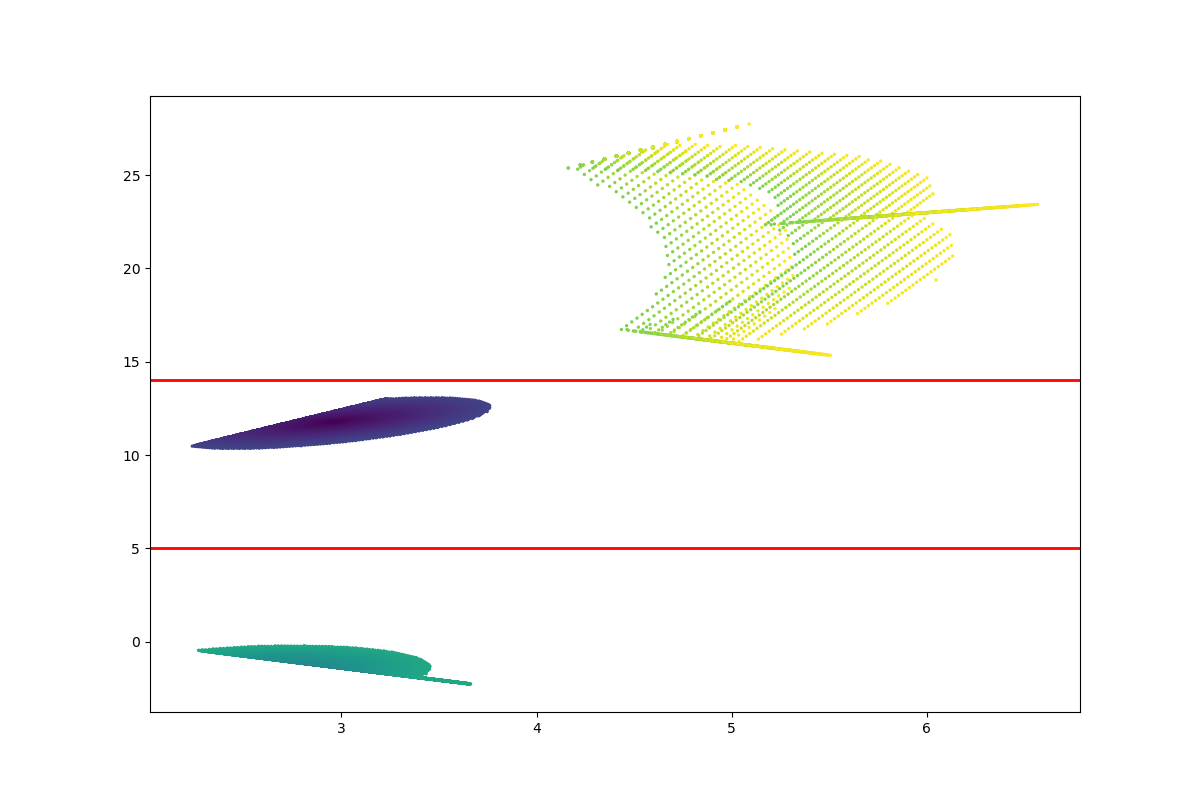}%
		}
		
		\caption{Experimental simulation for Example \ref{ex:toy_obstruction}. (a) The non-linearly separable topological arrangement in $\mathbb{R}^2$. (b) The datasets are linearly embedded into a higher-dimensional space $\mathbb{R}^3$. (c) Through successive hidden layers, the DNN progressively deforms the space, effectively lifting and bending the outer ring. (d) The final output configuration where all three datasets are strictly separable by hyperplanes.}
		\label{fig:sim_ex20_grid}
	\end{figure}
	
	We define the neural network architecture as $\Phi(x) = T_6 \circ \mathrm{Leaky\textup{-}ReLU}_{0.0001} \circ \cdots \circ \mathrm{Leaky\textup{-}ReLU}_{0.0001}  \circ T_1$, where the affine transformations  $T_i$ are parameterized by weight matrices $\{W_i\}_{i=1, 2, \cdots,  6}$ and bias vectors $\{b_i\}_{i=1, 2, \cdots, 6}$ given by
	\begin{align*}
		W_1 &= \begin{pmatrix}
			\frac{\sqrt{6}}{6} & \frac{\sqrt{2}}{2} \\
			\frac{\sqrt{6}}{6} & \frac{\sqrt{2}}{2} \\
			-\frac{\sqrt{6}}{3} & 0
		\end{pmatrix}, 
		& b_1 &= \begin{pmatrix}
			0 \\ 0 \\ 0
		\end{pmatrix}, \\
		W_2 &= \begin{pmatrix}
			0.8609 & 1.1220 & 0.7568\\
			-1.1220 & 0.8609 & 1.9176\\
			86.09 & 112.2 & 75.68
		\end{pmatrix}, 
		& b_2 &= \begin{pmatrix}
			2\\
			5\\
			-190
		\end{pmatrix},\\
		W_3 &= \begin{pmatrix}
			1 & 0 & 0\\
			0 & 1 & 0\\
			0 & 0 & -1
		\end{pmatrix}, 
		& b_3 &= \begin{pmatrix}
			0\\
			0\\
			10
		\end{pmatrix},\\
		W_4 &= \begin{pmatrix}
			1 & 0 & 0\\
			0 & -1 & 1\\
			0 & 0 & 0
		\end{pmatrix}, 
		& b_4 &= \begin{pmatrix}
			0\\
			-4.9\\
			0
		\end{pmatrix},\\
		W_5 &= \begin{pmatrix}
			1 & 0 & 0\\
			0 & -1 & 0\\
			0 & 0 & 1
		\end{pmatrix}, 
		& b_5 &= \begin{pmatrix}
			0\\
			0.001\\
			0
		\end{pmatrix},\\
		W_6 &= \begin{pmatrix}
			1 & 0 & 0\\
			0 & 10000 & 0
		\end{pmatrix}, 
		& b_6 &= \begin{pmatrix}
			0\\
			0
		\end{pmatrix}.
	\end{align*}
	By applying this transformation $\Phi$, the datasets $A$, $B$, and $C$ are successfully mapped into linearly separable configurations in the output space, as illustrated in Figure \ref{fig:sim_ex20_grid}. \hfill $\triangle$
	
	Now, we apply our framework to a topological unlinking problem. The Hopf Link consists of two linked circles which cannot be separated by ambient isotopy in $\mathbb{R}^3$ without intersecting. However, by utilizing an extra dimension and applying our relocation theory, we can resolve this topological obstruction by dimension lifting and projection.
	
	\begin{example}\label{ex:hopf_link}
		Let $K_1, K_2 \subset \mathbb{R}^3$ be two disjoint subsets, each diffeomorphic to $S^1$, forming a non-trivial Hopf Link, as shown in Figure \ref{fig:hopf_link}. Let $B_1, B_2 \subset \mathbb{R}^3$ be any two disjoint open balls. Then, there exists a Deep Neural Network (DNN) $\Phi: \mathbb{R}^3 \to \mathbb{R}^3$ using Leaky-ReLU, ELU, or SELU activation function with width $d=4$, such that
		\[
		\Phi(K_1) \subset B_1 \quad \textup{and} \quad \Phi(K_2) \subset B_2.
		\]
		Consequently, the Hopf Link can be made linearly separable by the DNN.
	\end{example}
	
	\begin{figure}[tbp]
		\centering
		\begin{tikzpicture}
			\begin{axis}[
				view={35}{25},      
				width=10cm,
				height=10cm,
				hide axis,         
				xmin=0, xmax=8,
				ymin=0, ymax=9,
				zmin=0, zmax=7,
				]
				
				\addplot3 [
				samples y=0,        
				domain=180:360,
				variable=\t,
				samples=60,
				color=orange!80,
				line width=6pt
				] ({4}, {5+2*cos(\t)}, {3+2*sin(\t)});
				
				\addplot3 [
				samples y=0,        
				domain=0:360,
				variable=\t,
				samples=100,
				color=violet!70,
				line width=6pt,
				preaction={draw, white, line width=10pt} 
				] ({4+2*cos(\t)}, {3+2*sin(\t)}, {3});
				
				\addplot3 [
				samples y=0,        
				domain=0:180,
				variable=\t,
				samples=60,
				color=orange!80,
				line width=6pt,
				preaction={draw, white, line width=10pt}
				] ({4}, {5+2*cos(\t)}, {3+2*sin(\t)});
				
				\node at (axis cs: 7.2, 3, 3) {\Large $\mathbf{K_1}$};
				\node at (axis cs: 4, 5, 5.8) {\Large $\mathbf{K_2}$};
				
			\end{axis}
		\end{tikzpicture}
		\caption{A non-trivial Hopf Link embedded in $\mathbb{R}^3$.}
		\label{fig:hopf_link}
	\end{figure}
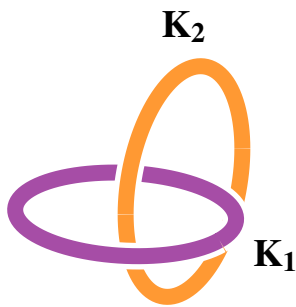
	
	\begin{proof}
		Let $D = K_1 \cup K_2$. We consider $\mathbb{R}^3$ as the hyperplane $\mathbb{R}^3 \times \{0\} \subset \mathbb{R}^4$.
		Since $K_1$ and $K_2$ are disjoint closed sets in $\mathbb{R}^3$, by the existence of smooth bump functions (Proposition 2.25 in Lee \cite{lee2013introduction}), there exists a smooth function $\psi: \mathbb{R}^3 \to [0, 1]$ such that $\psi|_{K_1} \equiv 0$ and $\psi|_{K_2} \equiv 1$.
		
		Define the map $E: \mathbb{R}^3 \to \mathbb{R}^4$ as
		\[
		E(x) = (x, C \cdot \psi(x)),
		\]
		where $C > 0$ is a sufficiently large constant. This map $E$ is a smooth embedding because its restriction to the first three coordinates is the identity map.
		
		Since $K_1$ and $K_2$ are compact, they are bounded. There exists a radius $R > 0$ such that $K_1, K_2 \subset B^3(0, R)$.
		In $\mathbb{R}^4$, the image sets are
		\begin{align*}
			E(K_1) = K_1 \times \{0\} \subset B^3(0, R) \times \{0\}, \\
			E(K_2) = K_2 \times \{C\} \subset B^3(0, R) \times \{C\}.
		\end{align*}
		Now, consider two balls in $\mathbb{R}^4$, given by
		\[
		\mathcal{U}_1 = B^4((0,0,0,0), 2R) \quad \textup{and} \quad \mathcal{U}_2 = B^4((0,0,0,C), 2R).
		\]
		Clearly, $E(K_1) \subset \mathcal{U}_1$ and $E(K_2) \subset \mathcal{U}_2$.
		By choosing $C > 4R$, we ensure that the balls $\mathcal{U}_1$ and $\mathcal{U}_2$ are mutually disjoint.
		
		Let $\proj: \mathbb{R}^4 \to \mathbb{R}^3$ be the standard projection. Choose two target balls $\mathcal{V}_1, \mathcal{V}_2 \subset \mathbb{R}^4$ such that they are mutually disjoint, disjoint from $\mathcal{U}_1$ and $\mathcal{U}_2$, and satisfy $\proj(\mathcal{V}_1) \subset B_1$ and $\proj(\mathcal{V}_2) \subset B_2$.
		
		In view of Theorem \ref{thm:disjoint_mapping}, there exists a global diffeomorphism $\Psi: \mathbb{R}^4 \to \mathbb{R}^4$ such that
		\[
		\Psi(E(K_i)) \subset \mathcal{V}_i, \quad i=1, 2.
		\]
		
		We define the mapping $g = \Psi \circ E: \mathbb{R}^3 \to \mathbb{R}^4$ and construct $f: D \to \mathbb{R}^3$ as the composition
		\[
		f(x) = \proj \circ g(x).
		\]
		Consequently, we have $f(K_i) = \proj \circ (\Psi(E(K_i))) \subset \proj(\mathcal{V}_i) \subset B_i$ for $i=1,2$.
		
		Since $f$ satisfies the condition of Theorem \ref{thm:projection} (with $m=4$), for any $\epsilon > 0$, there exists a width-$4$ DNN $\Phi: \mathbb{R}^3 \to \mathbb{R}^3$ with Leaky-ReLU or ELU activation function such that
		\[
		\lVert\Phi(x) - f(x)\rVert < \epsilon, \quad \forall x \in K_1 \cup K_2.
		\]
		
		Finally, let $\delta_1 = \dist(f(K_1), \partial B_1)$ and $\delta_2 = \dist(f(K_2), \partial B_2)$. Since $f(K_i)$ is compact and strictly contained in the open ball $B_i$, these distances are strictly positive.
		Choosing $\epsilon < \min(\delta_1, \delta_2)$, the approximation condition implies
		\[
		\Phi(K_1) \subset B_1 \quad \textup{and} \quad \Phi(K_2) \subset B_2.
		\]
		This completes the proof.
	\end{proof}
	
	In view of the above arguments, it is easy to see that a width-$4$ deep neural network (DNN) can map the Hopf Link to two linearly separable circles in the plane, as demonstrated in the following computer simulation.
	
	\noindent\textbf{Experimental Simulation.} To empirically demonstrate this theoretical result in Example \ref{ex:hopf_link}, we constructed a dataset representing the Hopf Link and applied the proposed dimension-lifting DNN. Specifically, we define the Hopf Link $L = L_1 \cup L_2$ as follows:
	\begin{align*}
		L_1 &= \{(x, y, 0) \mid (x+1)^2+y^2=4\}, \\
		L_2 &= \{(x, 0, z) \mid (x-1)^2+z^2=4\}.
	\end{align*}
	We trained a width-$4$ neural network expressed as $\Phi(x) = T_5 \circ \sigma \circ \dots \circ \sigma \circ T_1$, where the activation function $\sigma$ is $\mathrm{ELU}_1$. The weight matrices $\{W_i\}_{i=1, \dots, 5}$ and bias vectors $\{b_i\}_{i=1, \dots, 5}$ are given by
	\begin{align*}
		W_1 &= \begin{pmatrix}
			0.7202 & 0.2662 &  0.5537 \\
			-0.0776 &  0.4474 &  0.2805 \\
			-0.2377 &  0.1716 & 0.7948 \\
			0.3674 &  0.4314 & -0.3063
		\end{pmatrix}, 
		& b_1 &= \begin{pmatrix}
			-0.7492 \\ -0.2611 \\ -0.7549 \\ 1.0830
		\end{pmatrix}, \\
		W_2 &= \begin{pmatrix}
			-0.7003 & -1.9739 &   0.7907 &  0.6326 \\
			2.2762 & -2.6667 &  2.2235 & 0.6519 \\
			1.2144 &  2.4500 &  -0.7188 & -1.2508 \\
			0.2836 &  0.4959 & -1.5312 & -0.5781
		\end{pmatrix},
		& b_2 &= \begin{pmatrix}
			0.8253 \\ -0.1560 \\ 0.9420 \\ 0.3498
		\end{pmatrix}, \\
		W_3 &= \begin{pmatrix}
			0.1789 & -0.4813 & -0.2008 &  1.2680 \\
			-0.5009 & -0.3258 & -1.3531 & -1.0667 \\
			-0.7222 & -0.4798 & -2.3012 &  0.9997 \\
			1.1325 &  0.2196 &  0.1678 &  0.5455
		\end{pmatrix},
		& b_3 &= \begin{pmatrix}
			0.7573 \\  0.5837 \\ -0.2989 \\ 0.5017
		\end{pmatrix}, \\
		W_4 &= \begin{pmatrix}
			-0.7919 &  1.8838 &  0.1805 &  0.1470 \\
			0.5050 &  1.8152 &  1.5224 & -0.2909 \\
			-1.3450 & -2.1239 & -0.3029 &  0.7759 \\
			-1.2642 &  2.9938 &  -0.5644 &  1.1552
		\end{pmatrix},
		& b_4 &= \begin{pmatrix}
			0.6058 \\ -1.0162 \\  0.8334 \\ -1.3651
		\end{pmatrix}, \\
		W_5 &= \begin{pmatrix}
			0.2904 &  4.5729 &  2.0251 &  3.5745 \\
			4.1556 & -0.6492 &  0.4879 & -2.0542
		\end{pmatrix},
		& b_5 &= \begin{pmatrix}
			-2.3310 \\ -0.5897
		\end{pmatrix}.
	\end{align*}
	\begin{figure}[tbp]
		\centering
		\subfloat[The original Hopf Link embedded in $\mathbb{R}^3$\label{fig:hopf_left}]{%
			\includegraphics[width=0.45\textwidth]{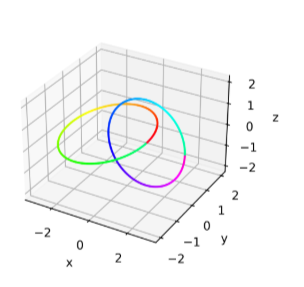}%
		}\hfill 
		\subfloat[The datasets effectively untangled and separated\label{fig:hopf_right}]{%
			\includegraphics[width=0.45\textwidth]{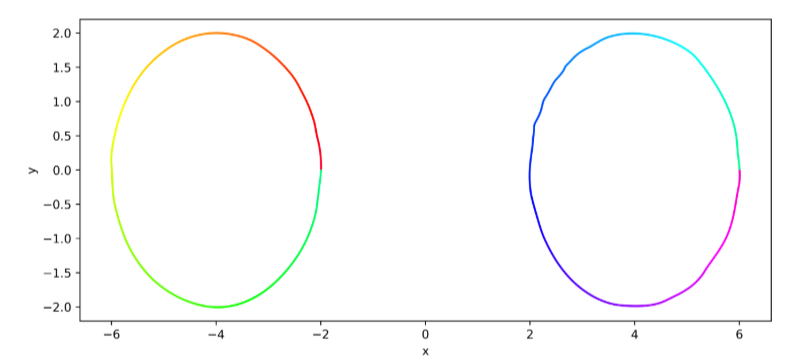}%
		}
		\caption{Simulation results for resolving the Hopf Link topological obstruction. \textbf{Left:} The original Hopf Link embedded in $\mathbb{R}^3$. \textbf{Right:} The datasets effectively untangled and separated into linearly classifiable configurations after the dimension-lifting DNN transformation.}
		\label{fig:sim_ex21_hopf}
	\end{figure}
	
	By applying this transformation, the originally entangled Hopf Link $L$ is successfully separated into two distinct, linearly classifiable circles, as illustrated in Figure \ref{fig:sim_ex21_hopf}. \hfill $\triangle$
	
	The dimension-lifting strategy demonstrated in Example \ref{ex:hopf_link} can be naturally generalized to arbitrary dimensions. Specifically, in view of the geometric embedding established in Theorem \ref{thm:embedding_rn1} and the universal approximation property of Theorem \ref{thm:nn_is_invertible_approximator}, we have the following general result.
	
	\begin{theorem}\label{thm:dnn_rn1}
		For arbitrary mutually disjoint compact datasets $K_1, \dots, K_m$ in $\mathbb{R}^n$, there exists a width-$(n+1)$ deep neural network (DNN) $\Phi$ with Leaky-ReLU, ELU, or SELU activation function such that $\Phi$ can make $K_1, \dots, K_m$ linearly separable in $\mathbb{R}^{n+1}$.
	\end{theorem}
	
	As a by-product, we consider the well-known ``Swiss Roll'' dataset to illustrate the practical implication of Theorem \ref{thm:embedding_approx}. We demonstrate that a width-$3$ DNN can approximate the smooth embedding that unrolls this 2-dimensional structure embedded in $\mathbb{R}^3$.
	
	\begin{example}\label{ex:swiss_roll}
		Let $S \subset \mathbb{R}^3$ be the ``Swiss Roll'' manifold. As a concrete example, one such manifold can be explicitly parameterized as the image of the smooth embedding $\psi: [0, 12] \times [T_0, T_1] \to \mathbb{R}^3$ given by
		\[
		\psi(s, t) = (s, t \cos t + 15, t \sin t + 15),
		\]
		where $0 < T_0 < T_1$ determines the rolling extent. As shown in Figure \ref{fig:swiss_roll}, there exists a width-$3$ Deep Neural Network (DNN) $\Phi$ using Leaky-ReLU, ELU, or SELU activation function that can ``roll out'' $S$.
	\end{example}
	
	\begin{figure}[tbp]
		\centering 
		\includegraphics[width=0.8\textwidth]{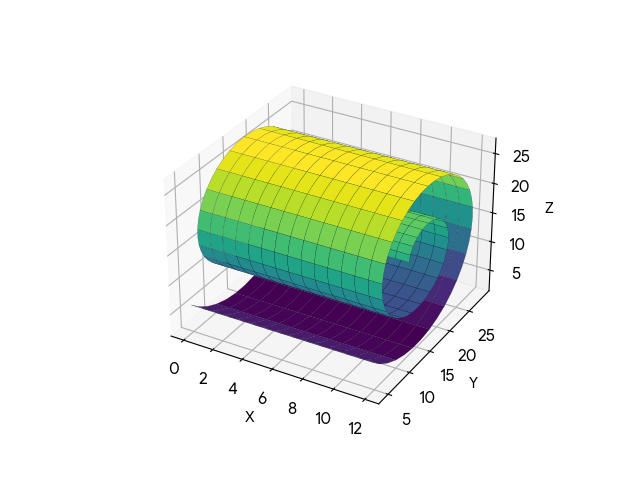} 
		\caption{The continuous 2-dimensional Swiss Roll manifold $S$ embedded in $\mathbb{R}^3$.} 
		\label{fig:swiss_roll}
	\end{figure}
	
	\begin{proof}		
		By definition, the Swiss Roll $S$ is a 2-dimensional differentiable ball embedded in $\mathbb{R}^3$. The unrolling map $\psi: S \to \mathbb{R}^3$ is a smooth embedding. Thus, $S$ and $\psi$ satisfy the hypotheses of Theorem \ref{thm:embedding_approx} with $n=3$ and $k=2$. Consequently, for any given $\epsilon > 0$, there exists a width-$3$ Deep Neural Network $\Phi: \mathbb{R}^3 \to \mathbb{R}^3$ such that
		\[
		\lVert \Phi(x) - \psi(x) \rVert < \epsilon, \quad \forall x \in S.
		\]
		This completes the proof.
	\end{proof}
	
\section*{Acknowledgements}
		We thank the School of Mathematics and Statistics, Huazhong University of Science and Technology (Wuhan, 430074, China), and the Hubei Key Laboratory of Engineering Modeling and Scientific Computing (Wuhan, 430074, China) for their support.
	
\section*{Funding}
		This work was supported by the National Natural Science Foundation of China (Grant No. 12531017).
	


\begin{thebibliography}{10}
\providecommand{\url}[1]{\texttt{#1}}
\providecommand{\urlprefix}{URL }
\providecommand{\eprint}[2][]{\url{#2}}

\bibitem{braga2020fundamentals}
Braga-Neto, U.: Fundamentals of Pattern Recognition and Machine Learning. Springer, Cham (2020)

\bibitem{cohen2020separability}
Cohen, U., Chung, S., Lee, D.~D., Sompolinsky, H.: Separability and geometry of object manifolds in deep neural networks. Nat. Commun. \textbf{11}, 746 (2020)

\bibitem{grootswagers2019untangling}
Grootswagers, T., Robinson, A.~K., Shatek, S.~M., Carlson, T.~A.: Untangling featural and conceptual object representations. NeuroImage \textbf{202}, 116083 (2019)

\bibitem{hanin2017approximating}
Hanin, B., Sellke, M.: Approximating continuous functions by {ReLU} nets of minimal width. arXiv preprint arXiv:1710.11278  (2017)

\bibitem{hwang2023minimum}
Hwang, G.: Minimum width for deep, narrow {MLP}: A diffeomorphism approach. In: Advances in Neural Information Processing Systems, 38 (2025)

\bibitem{kidger2020universal}
Kidger, P., Lyons, T.: Universal approximation with deep narrow networks. Conference on Learning Theory 2306--2327 (2020)

\bibitem{lee2013introduction}
Lee, J.~M.: Introduction to Smooth Manifolds. Second ed., Graduate Texts in Mathematics 218, Springer, New York (2013)

\bibitem{palais1960extending}
Palais, R.~S.: Extending diffeomorphisms. Proceedings of the American Mathematical Society \textbf{11}, 274--277 (1960)

\bibitem{teshima2020coupling}
Teshima, T., Ishikawa, I., Tojo, K., Oono, K., Ikeda, M., Sugiyama, M.: Coupling-based invertible neural networks are universal diffeomorphism approximators. Advances in Neural Information Processing Systems \textbf{33}, 3362--3373 (2020)

\bibitem{yang2025minimum}
Yang, X.-S., Zhou, Q., Zhou, X.: Minimum width of deep narrow networks for universal approximation. arXiv preprint arXiv:2511.06837  (2025)

\end{thebibliography}
\end{document}